\documentclass{article}
\usepackage{proceed2e} 

\usepackage[english]{babel}
\usepackage[utf8]{inputenc}

\usepackage[unicode]{hyperref}
\usepackage{enumerate}
\usepackage{amsmath}
\usepackage{times}
\usepackage{amssymb}
\usepackage{mathtools}
\usepackage{xtab} 
\usepackage{array} 
\usepackage{tikz}

\usepackage{algorithm}
\usepackage[noend]{algpseudocode}

\usepackage{mythm}
\usepackage{aixi}

\def\eps{\varepsilon}

\def\BayesExp{\mbox{BayesExp}}

\newcommand{\assref}[1]{\hyperref[#1]{\autoref*{ass:gamma}\ref*{#1}}}

\usepackage{etoolbox}
\AtEndEnvironment{example}{
	\qed
}

\usepackage[explicit]{titlesec}
\titleformat{\section}{\large\bf\raggedright}{\thesection}{1em}{\MakeUppercase{#1}}
\titleformat{\subsection}{\normalsize\bf\raggedright}{\thesubsection}{1em}{\MakeUppercase{#1}}

\AtBeginDocument{

}

\title{Thompson Sampling is Asymptotically Optimal in General Environments}
\author{
	{\bf Jan Leike} \\
	Australian National University \\
	\href{mailto:jan.leike@anu.edu.au}{jan.leike@anu.edu.au}
	\And
	{\bf Tor Lattimore} \\
	University of Alberta \\
	\href{mailto:tor.lattimore@gmail.com}{tor.lattimore@gmail.com}
	\And
	{\bf Laurent Orseau} \\
	Google DeepMind \\
	\href{mailto:lorseau@google.com}{lorseau@google.com}
	\And
	{\bf Marcus Hutter} \\
	Australian National University \\
	\href{mailto:marcus.hutter@anu.edu.au}{marcus.hutter@anu.edu.au}
}

\begin{document}

\maketitle

\begin{abstract}%
We discuss a variant of Thompson sampling for
nonparametric reinforcement learning in
a countable classes of general stochastic environments.
These environments can be non-Markov, non-ergodic, and partially observable.
We show that Thompson sampling learns the environment class
in the sense that
(1) asymptotically its value converges to the optimal value in mean and
(2) given a recoverability assumption regret is sublinear.
\end{abstract}


\section{Introduction}
\label{sec:introduction}

In reinforcement learning (RL)
an agent interacts with an unknown environment
with the goal of maximizing rewards.
Recently reinforcement learning has received a surge of interest,
triggered by its success in applications
such as simple video games~\cite{MKSRV+:2015deepQ}.
However, theory is lagging behind application
and most theoretical analyses has been done
in the bandit framework and
for Markov decision processes (MDPs).
These restricted environment classes fall short of
the full reinforcement learning problem and
theoretical results usually assume ergocity and
visiting every state infinitely often.
Needless to say, these assumptions are not satisfied
for any but the simplest applications.

Our goal is to lift these restrictions;
we consider \emph{general reinforcement learning},
a top-down approach to RL
with the aim to understand the fundamental underlying problems
in their generality.
Our approach to general RL is \emph{nonparametric}:
we only assume that
the true environment belongs to a given countable environment class.

We are interested in agents that maximize rewards \emph{optimally}.
Since the agent does not know the true environment in advance,
it is not obvious what optimality should mean.
We discuss two different notions of optimality:
\emph{asymptotic optimality} and \emph{worst-case regret}.

\emph{Asymptotic optimality}
requires that asymptotically the agent learns to act optimally,
i.e., that the discounted value of the agent's policy $\pi$
converges to the optimal discounted value,
$V^*_\mu - V^\pi_\mu \to 0$ for all environments $\mu$
from the environment class.
This convergence is impossible for deterministic policies
since the agent has to explore infinitely often and for long stretches of time,
but there are policies that converge almost surely in Cesàro average~\cite{LH:2011opt}.
Bayes-optimal agents are generally not asymptotically optimal%
~\cite{Orseau:2013}.
However, asymptotic optimality can be achieved through
an exploration component on top of a Bayes-optimal agent%
~\cite[Ch.~5]{Lattimore:2013} or through optimism~\cite{SH:2015opt}.

Asymptotic optimality in mean is essentially a weaker variant of
\emph{probably approximately correct} (PAC)
that comes without a concrete convergence rate:
for all $\varepsilon > 0$ and $\delta > 0$
the probability that our policy is $\varepsilon$-suboptimal
converges to zero (at an unknown rate).
Eventually this probability will be less than $\delta$.
Since our environment class can be very large and non-compact,
concrete PAC/convergence rates are likely impossible.

\emph{Regret} is how many expected rewards the agent forfeits by
not following the best informed policy.
Different problem classes have different regret rates,
depending on the structure and the difficulty of the problem class.
Multi-armed bandits provide a (problem-independent)
worst-case regret bound of $\Omega(\sqrt{KT})$
where $K$ is the number of arms~\cite{BCB:2012bandits}.
In Markov decision processes (MDPs)
the lower bound is $\Omega(\sqrt{DSAT})$
where $S$ is the number of states, $A$ the number of actions, and
$D$ the diameter of the MDP~\cite{AJO:2009MDPs}.
For a countable class of environments given by
state representation functions that map histories to MDP states,
a regret of $\tilde O(T^{2/3})$ is achievable
assuming the resulting MDP is weakly communicating~\cite{NMRO:2013}.
A problem class is considered \emph{learnable}
if there is an algorithm that has a sublinear regret guarantee.

This paper continues a narrative that started
with definition of the Bayesian agent AIXI~\cite{Hutter:2000} and
the proof that it satisfies various optimality guarantees~\cite{Hutter:2002}.
Recently it was revealed that
these optimality notions are trivial or subjective~\cite{LH:2015priors}:
a Bayesian agent does not explore enough to lose the prior's bias,
and a particularly bad prior can make the agent conform to
any arbitrarily bad policy as long as this policy yields
some rewards.
These negative results put the Bayesian approach to (general) RL into question.
In this paper we remedy the situation
by showing that using Bayesian techniques an agent
can indeed be optimal in an objective sense.

The agent we consider is known as \emph{Thompson sampling},
\emph{posterior sampling}, or
the \emph{Bayesian control rule}~\cite{Thompson:1933}.
It samples an environment $\rho$ from the posterior,
follows the $\rho$-optimal policy for one effective horizon
(a lookahead long enough to encompass most of the discount function's mass),
and then repeats.
We show that this agent's policy is asymptotically optimal in mean
(and, equivalently, in probability).
Furthermore, using a recoverability assumption on the environment,
and some (minor) assumptions on the discount function,
we prove that the worst-case regret is sublinear.
This is the first time convergence and regret bounds
of Thompson sampling have been shown under
such general conditions.

Thompson sampling was originally proposed by Thompson as a bandit algorithm~\cite{Thompson:1933}.
It is easy to implement and often achieves quite good results%
~\cite{CH:2011Thompson}.
In multi-armed bandits
it attains optimal regret~\cite{AG:2011Thompson,KKM:2012Thompson}.
Thompson sampling has also been considered for MDPs:
as model-free method relying on distributions over $Q$-functions
with convergence guarantee~\cite{DFR:1998Thompson}, and
as a model-based algorithm without theoretical analysis~\cite{Strens:2000}.
Bayesian and frequentist regret bounds have also been established~\cite{ORR:2013thompson,OVR:2014,GS:2015thompson}.
PAC guarantees have been established
for an optimistic variant of Thompson sampling for MDPs%
~\cite{ALLNW:2009Thompson}.

For general RL Thompson sampling
was first suggested in \cite{OB:2010Thompson} with resampling at every time step.
The authors prove that the action probabilities of Thompson sampling
converge to the action probability of the optimal policy almost surely,
but require a finite environment class and
two (arguably quite strong)
technical assumptions on the behavior of the posterior distribution
(akin to ergodicity)
and the similarity of environments in the class.
Our convergence results do not require these assumptions,
but we rely on an (unavoidable) recoverability assumption for our regret bound.

\autoref{app:notation} contains a list of notation and
\autoref{app:omitted-proofs} contains omitted proofs.

\section{Preliminaries}
\label{sec:preliminaries}


The set $\X^* := \bigcup_{n=0}^\infty \X^n$ is
the set of all finite strings over the alphabet $\X$ and
the set $\X^\infty$ is the set of all infinite strings
over the alphabet $\X$.
The empty string is denoted by $\epsilon$, not to be confused
with the small positive real number $\eps$.
Given a string $x \in \X^*$, we denote its length by $|x|$.
For a (finite or infinite) string $x$ of length $\geq k$,
we denote with $x_{1:k}$ the first $k$ characters of $x$,
and with $x_{<k}$ the first $k - 1$ characters of $x$.

The notation $\Delta\mathcal{Y}$ denotes
the set of probability distributions over $\mathcal{Y}$.


In reinforcement learning,
an agent interacts with an environment in cycles:
at time step $t$ the agent chooses an \emph{action} $a_t \in \A$ and
receives a \emph{percept} $e_t = (o_t, r_t) \in \E$
consisting of an \emph{observation} $o_t \in \O$
and a real-valued \emph{reward} $r_t$;
the cycle then repeats for $t + 1$.
We assume that rewards are bounded between $0$ and $1$ and
that the set of actions $\A$ and the set of percepts $\E$ are finite.

We fix a \emph{discount function}
$\gamma: \mathbb{N} \to \mathbb{R}$ with
$\gamma_t \geq 0$ and $\sum_{t=1}^\infty \gamma_t < \infty$.
Our goal is to maximize discounted rewards $\sum_{t=1}^\infty \gamma_t r_t$.
The \emph{discount normalization factor} is defined as
$\Gamma_t := \sum_{k=t}^\infty \gamma_k$.
The \emph{effective horizon} $H_t(\eps)$ is a horizon
that is long enough to encompass all but an $\eps$
of the discount function's mass:
\begin{equation}\label{eq:effective-horizon}
H_t(\eps) := \min \{ k \mid \Gamma_{t+k} / \Gamma_t \leq \eps \}
\end{equation}

A \emph{history} is an element of $\H$.
We use $\ae \in \A \times \E$ to denote one interaction cycle,
and $\ae_{<t}$ to denote a history of length $t - 1$.
We treat action, percepts, and histories
both as outcomes and as random variables.
A \emph{policy} is a function $\pi: \H \to \Delta\A$
mapping a history $\ae_{<t}$ to a distribution over the actions
taken after seeing this history;
the probability of action $a$ is denoted $\pi(a \mid \ae_{<t})$.
An \emph{environment} is a function $\nu: \H \times \A \to \Delta\E$
mapping a history $\ae_{<t}$ and an action $a_t$ to
a distribution over the percepts generated after this history;
the probability of percept $e$ is denoted $\nu(e \mid \ae_{<t} a_t)$.

A policy $\pi$ and an environment $\nu$ generate
a probability measure $\nu^\pi$ over infinite histories $(\A \times \E)^\infty$,
defined by its values on the cylinder sets
$\{ h \in (\A \times \E)^\infty \mid h_{<t} = \ae_{<t} \}$:
\[
   \nu^\pi(\ae_{<t})
:= \prod_{k=1}^{t-1} \pi(a_k \mid \ae_{<k}) \nu(e_k \mid \ae_{<k} a_k)
\]

When we take an expectation $\EE^\pi_\nu$ of a random variable $X_t(\ae_{<t})$
this is to be understood as the expectation of the history $\ae_{<t}$
for a fixed time step $t$
drawn from $\nu^\pi$, i.e.,
\[
   \EE_\nu^\pi[ X_t(\ae_{<t}) ]
:= \sum_{\ae_{<t}} \nu^\pi(\ae_{<t}) X_t(\ae_{<t}).
\]
We often do not explicitly add the subscript $t$ to time-dependent random variables.

\begin{definition}[Value Function]
\label{def:value-function}
The \emph{value} of a policy $\pi$ in an environment $\nu$
given history $\ae_{<t}$ is defined as
\[
   V_\nu^\pi(\ae_{<t})
:= \frac{1}{\Gamma_t}
     \EE_\nu^\pi\left[ \sum_{k=t}^\infty \gamma_k r_k \,\middle|\, \ae_{<t} \right],
\]
if $\Gamma_t > 0$ and $V^\pi_\nu(\ae_{<t}) := 0$ if $\Gamma_t = 0$.
The \emph{optimal value} is defined as
$V^*_\nu(h) := \sup_\pi V^\pi_\nu(h)$.
\end{definition}

The normalization constant $1/\Gamma_t$
ensures that values are bounded between $0$ and $1$.
We also use the \emph{truncated value function}
\[
   V_\nu^{\pi,m}(\ae_{<t})
:= \frac{1}{\Gamma_t}
     \EE_\nu^\pi\left[ \sum_{k=t}^m \gamma_k r_k \,\middle|\, \ae_{<t} \right].
\]

For each environment $\mu$ there is an \emph{optimal policy} $\pi^*_\mu$
that takes an \emph{optimal action} for each history~\cite[Thm.\ 10]{LH:2014discounting}:
\[
\pi^*_\mu(a_t \mid \ae_{<t}) > 0
\;\Longrightarrow\;
a_t \in \argmax_a V^*_\mu(\ae_{<t} a)
\]

Let $\M$ denote a countable class of environments.
We assume that $\M$ is large enough to contain the true environment
(e.g.\ the class of all computable environments).
Let $w \in \Delta\M$ be a prior probability distribution on $\M$ and
let
\[
\xi := \sum_{\nu \in \M} w(\nu) \nu
\]
denote the corresponding Bayesian mixture over the class $\M$.
After observing the history $\ae_{<t}$
the prior $w$ is updated to the posterior
\[
w(\nu \mid \ae_{<t}) := w(\nu) \frac{\nu(\ae_{<t})}{\xi(\ae_{<t})}.
\]
We also use the notation
$w(\M' \mid \ae_{<t}) := \sum_{\nu \in \M'} w(\nu \mid \ae_{<t})$
for a set of environments $\M' \subseteq \M$.
Likewise we define
$\nu(A \mid \ae_{<t}) := \sum_{h \in A} \nu(h \mid \ae_{<t})$
for a prefix-free set of histories $A \subseteq \H$.


Let $\nu, \rho \in \M$ be two environments,
let $\pi_1, \pi_2$ be two policies, and
let $m \in \mathbb{N}$ be a lookahead time step.
The \emph{total variation distance} is defined as
\begin{align*}
D_m(\nu^{\pi_1}, \rho^{\pi_2} \mid \ae_{<t}) &:= \\
\sup_{A \subseteq (\A \times \E)^m} &\Big|
  \nu^{\pi_1}(A \mid \ae_{<t}) - \rho^{\pi_2}(A \mid \ae_{<t})
\Big|.
\end{align*}
with $D_\infty(\nu^{\pi_1}, \rho^{\pi_2} \mid \ae_{<t}) :=
\lim_{m \to \infty} D_m(\nu^{\pi_1}, \rho^{\pi_2} \mid \ae_{<t})$.


\begin{lemma}[Bounds on Value Difference]
\label{lem:value-bound}
For any policies $\pi_1$, $\pi_2$, any environments $\rho$ and $\nu$,
and any horizon $t \leq m \leq \infty$,
\[
       |V^{\pi_1,m}_\nu(\ae_{<t}) - V^{\pi_2,m}_\rho(\ae_{<t})|
~\leq~ D_m(\nu^{\pi_1}, \rho^{\pi_2} \mid \ae_{<t})
\]
\end{lemma}
\begin{proof}
See \autoref{app:omitted-proofs}.
\end{proof}

\section{Thompson Sampling is Asymptotically Optimal}
\label{sec:thompson-sampling}

Strens proposes following the optimal policy for one episode or
``related to the number of state transitions
the agent is likely to need to plan ahead''~\cite{Strens:2000}.
We follow Strens' suggestion and resample at the effective horizon.

Let $\eps_t$ be a monotone decreasing sequence of positive reals such that
$\eps_t \to 0$ as $t \to \infty$.
We define our Thompson-sampling policy $\pi_T$ in \autoref{alg:Thompson-sampling}.
\begin{algorithm}
\begin{center}
\begin{algorithmic}[1]
\While{true}
	\State sample $\rho \sim w(\,\cdot \mid \ae_{<t})$
	\State follow $\pi^*_\rho$ for $H_t(\eps_t)$ steps
\EndWhile
\end{algorithmic}
\end{center}
\caption{Thompson sampling policy $\pi_T$}
\label{alg:Thompson-sampling}
\end{algorithm}

Note that $\pi_T$ is a stochastic policy
since we occasionally sample from a distribution.
We assume that this sampling is independent of everything else.

\begin{definition}[Asymptotic Optimality]
\label{def:asymptotic-optimality}
A policy $\pi$ is
\emph{asymptotically optimal in an environment class $\M$}
iff for all $\mu \in \M$
\begin{equation}\label{eq:asymptotic-optimality}
V^*_\mu(\ae_{<t}) - V^\pi_\mu(\ae_{<t})
\to 0
\text{ as $t \to \infty$}
\end{equation}
on histories drawn from $\mu^\pi$.
\end{definition}

There are different types of asymptotic optimalities
based on the type of stochastic convergence in \eqref{eq:asymptotic-optimality}.
If this convergence occurs almost surely,
it is called \emph{strong asymptotic optimality}~\cite[Def.~7]{LH:2011opt};
if this convergence occurs in mean,
it is called \emph{asymptotic optimality in mean};
if this convergence occurs in probability,
it is called \emph{asymptotic optimality in probability}; and
if the Cesàro averages converge almost surely,
it is called \emph{weak asymptotic optimality}~\cite[Def.~7]{LH:2011opt}.

\subsection{Asymptotic Optimality in Mean}
\label{ssec:ao-in-mean}

This subsection is dedicated to proving the following theorem.

\begin{theorem}[Thompson Sampling is Asymptotically Optimal in Mean]
\label{thm:taxi-ao}
For all environments $\mu \in \M$,
\[
\EE_\mu^{\pi_T} \big[ V^*_\mu(\ae_{<t}) - V^{\pi_T}_\mu(\ae_{<t}) \big] \to 0
\text{ as $t \to \infty$.}
\]
\end{theorem}

This theorem immediately implies that
Thompson sampling is also asymptotically optimal in probability:
The convergence in mean of the random variables
$X_t := V^*_\mu(\ae_{<t}) - V^{\pi_T}_\mu(\ae_{<t})$
stated in \autoref{thm:taxi-ao}
is equivalent to convergence in probability in the sense that
$\mu^{\pi_T}[ X_t > \varepsilon ] \to 0$ as $t \to \infty$
for all $\varepsilon > 0$
because the random variables $X_t$ are nonnegative and bounded.
However, this does not imply almost sure convergence
(see \autoref{ssec:sao}).

Define the \emph{Bayes-expected total variation distance}
\[
   F^\pi_m(\ae_{<t})
:= \sum_{\rho \in \M} w(\rho \mid \ae_{<t}) D_m(\rho^\pi, \xi^\pi \mid \ae_{<t})
\]
for $m \leq \infty$.

If we replace the distance measure $D_m$ by cross-entropy, then
the quantity $F^\pi_m(\ae_{<t})$ becomes the
Bayes-expected \emph{information gain}~\cite[Eq.~3.5]{Lattimore:2013}.

For the proof of \autoref{thm:taxi-ao} we need the following lemma.

\begin{lemma}[F Vanishes On-Policy]
\label{lem:info-gain}
For any policy $\pi$ and any environment $\mu$,
\[
\EE_\mu^\pi[F^\pi_\infty(\ae_{<t})] \to 0
\text{ as $t \to \infty$.}
\]
\end{lemma}
\begin{proof}
See \autoref{app:omitted-proofs}.
\end{proof}

\begin{proof}[Proof of \autoref{thm:taxi-ao}]
Let $\beta, \delta > 0$ and
let $\eps_t > 0$ denote the sequence
used to define $\pi_T$ in \autoref{alg:Thompson-sampling}.
We assume that $t$ is large enough such that
$\eps_k \leq \beta$ for all $k \geq t$ and that
$\delta$ is small enough such that $w(\mu \mid \ae_{<t}) > 4\delta$ for all $t$,
which holds since $w(\mu \mid \ae_{<t}) \not\to 0$
$\mu^\pi$-almost surely for any policy $\pi$~\cite[Lem.~3i]{Hutter:2009MDL}.

The stochastic process $w(\nu \mid \ae_{<t})$
is a $\xi^{\pi_T}$-martingale since
\begingroup
\allowdisplaybreaks
\begin{align*}
&\phantom{=}~\,
   \EE_\xi^{\pi_T}[ w(\nu \mid \ae_{1:t}) \mid \ae_{<t} ] \\
&= \sum_{a_t e_t} \xi^{\pi_T}(\ae_t \mid \ae_{<t})
     w(\nu) \frac{\nu^{\pi_T}(\ae_{1:t})}{\xi^{\pi_T}(\ae_{1:t})} \\
&= \sum_{a_t e_t} \xi^{\pi_T}(\ae_t \mid \ae_{<t})
     w(\nu \mid \ae_{<t}) \frac{\nu^{\pi_T}(\ae_t \mid \ae_{<t})}{\xi^{\pi_T}(\ae_t \mid \ae_{<t})} \\
&= w(\nu \mid \ae_{<t}) \sum_{a_t e_t} \nu^{\pi_T}(\ae_t \mid \ae_{<t}) \\
&= w(\nu \mid \ae_{<t}).
\end{align*}
\endgroup
By the martingale convergence theorem~\cite[Thm.~5.2.8]{Durrett:2010}
$w(\nu \mid \ae_{<t})$ converges $\xi^{\pi_T}$-almost surely and
because $\xi^{\pi_T} \geq w(\mu) \mu^{\pi_T}$
it also converges $\mu^{\pi_T}$-almost surely.

We argue that we can choose $t_0$
to be one of $\pi_T$'s resampling time steps large enough such that
for all $t \geq t_0$
the following three events hold simultaneously
with $\mu^{\pi_T}$-probability at least $1 - \delta$.
\begin{enumerate}[(i)]
\item \label{itm:M'}
	There is a finite set $\M' \subset \M$ with
	$w(\M' \mid \ae_{<t}) > 1 - \delta$
	and $w(\nu \mid \ae_{<k}) \not\to 0$ as $k \to \infty$
	for all $\nu \in \M'$.
\item \label{itm:posterior-convergence}
	$\left| w(\M'' \mid \ae_{<t}) - w(\M'' \mid \ae_{<t_0}) \right|
	\leq \delta$ for all $\M'' \subseteq \M'$.
\item \label{itm:info-gain}
	$F^{\pi_T}_\infty(\ae_{<t}) < \delta \beta w_{\min}^2$.
\end{enumerate}
where
$w_{\min} := \inf \{ w(\nu \mid \ae_{<k}) \mid k \in \mathbb{N}, \nu \in \M' \}$,
which is positive by (\ref{itm:M'}).

(\ref{itm:M'}) and (\ref{itm:posterior-convergence}) are satisfied eventually
because the posterior $w(\,\cdot \mid \ae_{<t})$
converges $\mu^{\pi_T}$-almost surely.
Note that the set $\M'$ is random:
the limit of $w(\nu \mid \ae_{<t})$ as $t \to \infty$
depends on the history $\ae_{1:\infty}$.
Without loss of generality,
we assume the true environment $\mu$ is contained in $\M'$
since $w(\mu \mid \ae_{<t}) \not\to 0$ $\mu^{\pi_T}$-almost surely.
(\ref{itm:info-gain}) follows from \autoref{lem:info-gain}
since convergence in mean implies convergence in probability.

Moreover, we define the horizon $m := t + H_t(\eps_t)$
as the time step of the effective horizon at time step $t$.
Let $\ae_{<t}$ be a fixed history
for which (\ref{itm:M'}-\ref{itm:info-gain}) is satisfied.
Then we have
\begingroup
\allowdisplaybreaks
\begin{align*}
      \delta \beta w_{\min}^2
&>    F^{\pi_T}_\infty(\ae_{<t}) \\
&=    \sum_{\nu \in \M} w(\nu \mid \ae_{<t})
        D_\infty(\nu^{\pi_T}, \xi^{\pi_T} \mid \ae_{<t}) \\
&=    \mathbb{E}_{\nu \sim w(\,\cdot \mid \ae_{<t})} \left[
        D_\infty(\nu^{\pi_T}, \xi^{\pi_T} \mid \ae_{<t})
      \right] \\
&\geq \mathbb{E}_{\nu \sim w(\,\cdot \mid \ae_{<t})} \left[
        D_m(\nu^{\pi_T}, \xi^{\pi_T} \mid \ae_{<t})
      \right] \\
&\geq \beta w_{\min}^2 w( \M \setminus \M'' \mid \ae_{<t})
\end{align*}
\endgroup
by Markov's inequality where
\[
   \M''
:= \left\{ \nu \in \M \;\middle|\;
           D_m(\nu^{\pi_T}, \xi^{\pi_T} \mid \ae_{<t})
           < \beta w_{\min}^2
   \right\}.
\]
For our fixed history $\ae_{<t}$ we have
\begingroup
\allowdisplaybreaks
\begin{align*}
   1 - \delta
&< w( \M'' \mid \ae_{<t}) \\
&\stackrel{(\ref{itm:M'})}{\leq}
   w( \M'' \cap \M' \mid \ae_{<t}) + \delta \\
&\stackrel{(\ref{itm:posterior-convergence})}{\leq}
   w( \M'' \cap \M' \mid \ae_{<t_0}) + 2\delta \\
&\stackrel{(\ref{itm:M'})}{\leq}
   w( \M'' \mid \ae_{<t_0}) + 3\delta
\end{align*}
\endgroup
and thus we get
\begin{equation}\label{eq:D-bound}
  1 - 4\delta
< w \left[ D_m(\nu^{\pi_T}, \xi^{\pi_T} \mid \ae_{<t})
           < \beta w_{\min}^2
    \;\middle|\; \ae_{<t_0}
    \right].
\end{equation}
In particular, this bound holds for $\nu = \mu$
since $w(\mu \mid \ae_{<t_0}) > 4\delta$ by assumption.

It remains to show that with high probability the value $V^{\pi^*_\rho}_\mu$
of the sample $\rho$'s optimal policy $\pi^*_\rho$ is sufficiently close to
the $\mu$-optimal value $V^*_\mu$.
The worst case is that we draw the worst sample from $\M' \cap \M''$
twice in a row.
From now on, let $\rho$ denote the sample environment we draw at time step $t_0$,
and let $t$ denote some time step between
$t_0$ and $t_1 := t_0 + H_{t_0}(\eps_{t_0})$
(before the next resampling).
With probability $w(\nu' \mid \ae_{<t_0}) w(\nu' \mid \ae_{<t_1})$
we sample $\nu'$ both at $t_0$ and $t_1$
when following $\pi_T$.
Therefore we have for all $\ae_{t:m}$ and all $\nu \in \M$
\begin{align*}
&\phantom{=}~\,
      \nu^{\pi_T}(\ae_{1:m} \mid \ae_{<t}) \\
&\geq w(\nu' \mid \ae_{<t_0}) w(\nu' \mid \ae_{<t_1})
      \nu^{\pi^*_{\nu'}}(\ae_{1:m} \mid \ae_{<t}).
\end{align*}
Thus we get for all $\nu \in \M'$ (in particular $\rho$ and $\mu$)
\begingroup
\allowdisplaybreaks
\begin{align*}
&\phantom{=}~\,
      D_m(\mu^{\pi_T}, \rho^{\pi_T} \mid \ae_{<t}) \\
&\geq \sup_{\nu' \in \M} \sup_{A \subseteq (\A \times \E)^m}
      \Big| w(\nu' \mid \ae_{<t_0}) w(\nu' \mid \ae_{<t_1}) \\
&\qquad\qquad (\mu^{\pi^*_{\nu'}}(A \mid \ae_{<t}) - \rho^{\pi^*_{\nu'}}(A \mid \ae_{<t})) \Big| \\
&\geq w(\nu \mid \ae_{<t_0}) w(\nu \mid \ae_{<t_1}) \\
&\qquad\qquad \sup_{A \subseteq (\A \times \E)^m} \Big|
        \mu^{\pi^*_\nu}(A \mid \ae_{<t}) - \rho^{\pi^*_\nu}(A \mid \ae_{<t})
      \Big| \\
&\geq w_{\min}^2 D_m(\mu^{\pi^*_\nu}, \rho^{\pi^*_\nu} \mid \ae_{<t}).
\end{align*}
\endgroup
For $\rho \in \M''$ we get
\begin{align*}
&\phantom{=}~\,
      D_m(\mu^{\pi_T}, \rho^{\pi_T} \mid \ae_{<t}) \\
&\leq D_m(\mu^{\pi_T}, \xi^{\pi_T} \mid \ae_{<t}) + D_m(\rho^{\pi_T}, \xi^{\pi_T} \mid \ae_{<t}) \\
&\stackrel{\eqref{eq:D-bound}}{<}
      \beta w_{\min}^2 + \beta w_{\min}^2
 =    2\beta w_{\min}^2,
\end{align*}
which implies together with \autoref{lem:value-bound}
and the fact that rewards in $[0, 1]$
\begingroup
\allowdisplaybreaks
\begin{align*}
&\phantom{=}~\,
     \left| V^{\pi^*_\nu}_\mu(\ae_{<t}) - V^{\pi^*_\nu}_\rho(\ae_{<t}) \right| \\
&\leq \frac{\Gamma_{t+H_t(\eps_t)}}{\Gamma_t} +
      \left| V^{\pi^*_\nu,m}_\mu(\ae_{<t}) - V^{\pi^*_\nu,m}_\rho(\ae_{<t}) \right| \\
&\leq \eps_t + D_m(\mu^{\pi^*_\nu}, \rho^{\pi^*_\nu} \mid \ae_{<t}) \\
&\leq \eps_t + \tfrac{1}{w_{\min}^2} D_m(\mu^{\pi_T}, \rho^{\pi_T} \mid \ae_{<t}) \\
&< \beta + 2\beta
=     3\beta.
\end{align*}
\endgroup
Hence we get (omitting history arguments $\ae_{<t}$ for simplicity)
\begin{equation}\label{eq:V-beta-bound}
\begin{aligned}
      V^*_\mu
&=    V^{\pi^*_\mu}_\mu
 <    V^{\pi^*_\mu}_\rho + 3\beta
 \leq V^*_\rho + 3\beta \\
&=    V^{\pi^*_\rho}_\rho + 3\beta
 <    V^{\pi^*_\rho}_\mu + 3\beta + 3\beta
 =    V^{\pi^*_\rho}_\mu + 6\beta.
\end{aligned}
\end{equation}

With $\mu^{\pi_T}$-probability at least $1 - \delta$
(\ref{itm:M'}), (\ref{itm:posterior-convergence}), and (\ref{itm:info-gain})
are true,
with $\mu^{\pi_T}$-probability at least $1 - \delta$
our sample $\rho$ happens to be in $\M'$ by (\ref{itm:M'}), and
with $w(\,\cdot \mid \ae_{<t_0})$-probability at least $1 - 4\delta$
the sample is in $\M''$ by \eqref{eq:D-bound}.
All of these events are true simultaneously with probability at least
$1 - (\delta + \delta + 4\delta) = 1 - 6\delta$.
Hence the bound \eqref{eq:V-beta-bound} transfers for $\pi_T$ such that
with $\mu^{\pi_T}$-probability $\geq 1 - 6\delta$ we have
\[
V^*_\mu(\ae_{<t}) - V^{\pi_T}_\mu(\ae_{<t}) < 6\beta.
\]
Therefore
$\mu^{\pi_T}[ V^*_\mu(\ae_{<t}) - V^{\pi_T}_\mu(\ae_{<t}) \geq 6\beta ] < 6\delta$
and with $\delta \to 0$ we get that
$V_\mu^*(\ae_{<t}) - V_\mu^{\pi_T}(\ae_{<t}) \to 0$ as $t \to \infty$ in probability.
The value function is bounded, thus it also converges in mean
by the dominated convergence theorem.
\end{proof}

\subsection{Weak Asymptotic Optimality}
\label{ssec:wao}

It might appear that convergence in mean is more natural than
the convergence of Cesàro averages of weak asymptotic optimality.
However, both notions are not so fundamentally different
because they both allow an infinite number of bad mistakes
(actions that lead to $V^*_\mu - V^\pi_\mu$ being large).
Asymptotic optimality in mean allows bad mistakes
as long as their probability converges to zero;
weak asymptotic optimality allows bad mistakes
as long as the total time spent on bad mistakes grows sublinearly.

Lattimore and Hutter show that weak asymptotic optimality is possible
in a countable class of deterministic environments using an MDL-agent that
explores through bursts of random walks~\cite[Def.~10]{LH:2011opt}.
For classes of stochastic environments,
{\BayesExp} is weakly asymptotically optimal~\cite[Ch.~5]{Lattimore:2013}.
However, this requires the additional condition that
the effective horizon grows sublinearly, $H_t(\varepsilon_t) \in o(t)$,
while \autoref{thm:taxi-ao} does not require this condition.

Generally, weak asymptotic optimality and asymptotic optimality in mean
are incomparable
because the notions of convergence are incomparable for (bounded)
random variables.
First,
for deterministic sequences
(i.e.\ deterministic policies in deterministic environments),
convergence in mean is equivalent to (regular) convergence,
which implies convergence in Cesàro average, but not vice versa.
Second,
convergence in probability (and hence convergence in mean for bounded random variables)
does not imply almost sure convergence of Cesàro averages%
~\cite[Sec.~14.18]{Stoyanov:2013counterexamples}.
We leave open the question whether
the policy $\pi_T$ is weakly asymptotically optimal.

\subsection{Strong Asymptotic Optimality}
\label{ssec:sao}

Strong asymptotic optimality is known to be impossible for deterministic policies%
~\cite[Thm.~8.1]{LH:2011opt},
but whether it is possible for stochastic policies is an open question.
However, we show that Thompson sampling is not strongly asymptotically optimal.

\begin{example}[Thompson Sampling is not Strongly Asymptotically Optimal]
\label{ex:TS-not-sao}
Define $\A := \{ \alpha, \beta \}$, $\E := \{ 0, 1/2, 1 \}$, and
assume geometric discounting, $\gamma_t := \gamma^t$ for $\gamma \in (0, 1)$.
Consider the following class of environments
$\M := \{ \nu_\infty, \nu_1, \nu_2, \ldots \}$
(transitions are labeled with action, reward):
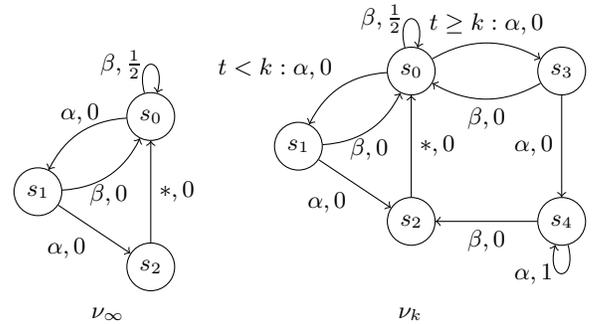
\begin{center}
\small
\begin{tabular}{cc}
\begin{tikzpicture}
\node[circle, draw, minimum height=2em] (s0) at (0, 0) {$s_0$};
\node[circle, draw, minimum height=2em] (s1) at (-1.5, -1) {$s_1$};
\node[circle, draw, minimum height=2em] (s2) at (0, -2) {$s_2$};

\draw[->] (s0) to[loop above] node[left] {$\beta, \frac{1}{2}$} (s0);
\draw[->] (s0) to[bend right] node[above] {$\alpha, 0$} (s1);
\draw[->] (s1) to[bend right] node[below] {$\beta, 0$} (s0);
\draw[->] (s1) to node[below left] {$\alpha, 0$} (s2);
\draw[->] (s2) to node[right] {$\ast, 0$} (s0);
\end{tikzpicture} & \hspace{-4mm}
\begin{tikzpicture}
\node[circle, draw, minimum height=2em] (s0) at (0, 0) {$s_0$};
\node[circle, draw, minimum height=2em] (s1) at (-1.5, -1) {$s_1$};
\node[circle, draw, minimum height=2em] (s2) at (0, -2) {$s_2$};
\node[circle, draw, minimum height=2em] (s3) at (2, 0) {$s_3$};
\node[circle, draw, minimum height=2em] (s4) at (2, -2) {$s_4$};

\draw[->] (s0) to[loop above] node[left] {$\beta, \frac{1}{2}$} (s0);
\draw[->] (s0) to[bend right] node[above left] {$t < k: \alpha, 0$} (s1);
\draw[->] (s1) to[bend right] node[below] {$\beta, 0$} (s0);
\draw[->] (s1) to node[below left] {$\alpha, 0$} (s2);
\draw[->] (s2) to node[right] {$\ast, 0$} (s0);
\draw[->] (s0) to[bend left] node[above] {$t \geq k: \alpha, 0$} (s3);
\draw[->] (s3) to node[left] {$\alpha, 0$} (s4);
\draw[->] (s3) to[bend left] node[below] {$\beta, 0$} (s0);
\draw[->] (s4) to[loop below] node[left] {$\alpha, 1$} (s4);
\draw[->] (s4) to node[below] {$\beta, 0$} (s2);
\end{tikzpicture} \\
$\nu_\infty$ & $\nu_k$
\end{tabular}
\end{center}
Environment $\nu_k$ works just like environment $\nu_\infty$
except that after time step $k$, the path to state $s_3$ gets unlocked and
the optimal policy is to take action $\alpha$ twice from state $s_0$.
The class $\M$ is a class of deterministic weakly communicating MDPs
(but as an MDP $\nu_k$ has more than 5 states).
The optimal policy in environment $\nu_\infty$ is to always take action $\beta$,
the optimal policy for environment $\nu_k$ is
to take action $\beta$ for $t < k$ and then
take action $\beta$ in state $s_1$ and action $\alpha$ otherwise.

Suppose the policy $\pi_T$ is acting in environment $\nu_\infty$.
Since it is asymptotically optimal in the class $\M$,
it has to take actions $\alpha\alpha$ from $s_0$ infinitely often:
for $t < k$ environment $\nu_k$ is indistinguishable from $\nu_\infty$,
so the posterior for $\nu_k$ is larger or equal to the prior.
Hence there is always a constant chance of sampling $\nu_k$
until taking actions $\alpha\alpha$,
at which point all environments $\nu_k$ for $k \leq t$ become falsified.

If the policy $\pi_T$ decides to explore and take the first action $\alpha$,
it will be in state $s_1$.
Let $\ae_{<t}$ denote the current history.
Then the $\nu_\infty$-optimal action is $\beta$ and
\[
  V^*_{\nu_\infty}(\ae_{<t})
= (1 - \gamma) \left(
    0 + \gamma\frac{1}{2} + \gamma^2 \frac{1}{2} + \ldots
  \right)
= \frac{\gamma}{2}.
\]
The next action taken by $\pi_T$ is $\alpha$ since
any optimal policy for any sampled environment
that takes action $\alpha$ once, takes that action again
(and we are following that policy for an $\eps_t$-effective horizon).
Hence
\[
     V^{\pi_T}_{\nu_\infty}(\ae_{<t})
\leq (1 - \gamma) \left(
       0 + 0 + \gamma^2\frac{1}{2} + \gamma^3 \frac{1}{2} + \ldots
     \right)
=    \frac{\gamma^2}{2}.
\]
Therefore $V^*_{\nu_\infty} - V^{\pi_T}_{\nu_\infty} \geq (\gamma - \gamma^2)/2 > 0$.
This happens infinitely often with probability one and
thus we cannot get almost sure convergence.
\end{example}

We expect that strong asymptotic optimality can be achieved
with Thompson sampling by resampling at every time step
(with strong assumptions on the discount function).

\section{Regret}
\label{sec:regret}

\subsection{Setup}
\label{ssec:regret-setup}

In general environments classes
worst-case regret is linear
because the agent can get caught in a trap and be unable to recover%
~\cite[Sec.~5.3.2]{Hutter:2005}.
To achieve sublinear regret we need to ensure that
the agent can recover from mistakes.
Formally, we make the following assumption.

\begin{definition}[Recoverability]
\label{def:recoverability}
An environment $\nu$ satisfies the \emph{recoverability assumption} iff
\[
\sup_\pi \left|
  \EE^{\pi^*_\nu}_\nu[ V^*_\nu(\ae_{<t}) ]
  - \EE^\pi_\nu[ V^*_\nu(\ae_{<t}) ]
\right|
\to 0 \text{ as $t \to \infty$}.
\]
\end{definition}
Recoverability compares following the worst policy $\pi$ for $t - 1$ time steps
and then switching to the optimal policy $\pi^*_\nu$ to
having followed $\pi^*_\nu$ from the beginning.
The recoverability assumption states that
switching to the optimal policy at any time step
enables the recovery of most of the value.

Note that \autoref{def:recoverability} demands that
it becomes less costly to recover from mistakes as time progresses.
This should be regarded as an effect of the discount function:
if the (effective) horizon grows,
recovery becomes easier because
the optimal policy has more time to perform a recovery.
Moreover, recoverability is on the optimal policy,
in contrast to the notion of ergodicity in MDPs
which demands returning to a starting state regardless of the policy.

\begin{remark}[Weakly Communicating POMDPs are Recoverable]
\label{rem:recoverable-POMDPs}
If the effective horizon is growing,
$H_t(\eps) \to \infty$ as $t \to \infty$,
then any weakly communicating finite state partially observable MDP
satisfies the recoverability assumption.
\end{remark}

\begin{definition}[Regret]
\label{def:regret}
The \emph{regret} of a policy $\pi$ in environment $\mu$ is
\[
   R_m(\pi, \mu)
:= \sup_{\pi'} \EE_\mu^{\pi'} \left[ \sum_{t=1}^m r_t \right]
   - \EE_\mu^\pi \left[ \sum_{t=1}^m r_t \right].
\]
\end{definition}

Note that regret is undiscounted and always nonnegative.
Moreover, the supremum is always attained by some policy
(not necessarily the ($V_\mu$-)optimal policy $\pi^*_\mu$
because that policy uses discounting),
since the space of possible different policies for the first $m$
actions is finite since we assumed
the set of actions $\A$ and the set of percepts $\E$ to be finite.

\begin{assumption}[Discount Function]
\label{ass:gamma}
Let the discount function $\gamma$ be such that
\begin{enumerate}[(a)]
\item \label{itm:gamma(a)}
	$\gamma_t > 0$ for all $t$,
\item \label{itm:gamma(b)}
	$\gamma_t$ is monotone decreasing in $t$, and
\item \label{itm:gamma(c)}
	$H_t(\eps) \in o(t)$ for all $\eps > 0$.
\end{enumerate}
\end{assumption}

This assumption demands that the discount function is somewhat well-behaved:
the function has no oscillations, does not become $0$, and
the horizon is not growing too fast.

\autoref{ass:gamma} is satisfied by geometric discounting:
$\gamma_t := \gamma^t > 0$~(\ref{itm:gamma(a)})
for some fixed constant $\gamma \in (0, 1)$ is monotone decreasing~(\ref{itm:gamma(b)}),
$\Gamma_t = \gamma^t / (1 - \gamma)$, and
$H_t(\eps) = \lceil \log_\gamma \eps \rceil \in o(t)$~(\ref{itm:gamma(c)}).

The problem with geometric discounting is that
it makes the recoverability assumption very strong:
since the horizon is not growing, the environment has to enable
\emph{faster recovery} as time progresses;
in this case weakly communicating partially observable MDPs
are \emph{not} recoverable.

A choice with $H_t(\eps) \to \infty$ that satisfies \autoref{ass:gamma} is
$\gamma_t := e^{-\sqrt{t}} / \sqrt{t}$~\cite[Sec.~2.3.1]{Lattimore:2013}.
For this discount function
$\Gamma_t \approx 2e^{-\sqrt{t}}$,
$H_t(\eps) \approx -\sqrt{t}\log \eps + (\log \eps)^2 \in o(t)$, and thus
$H_t(\eps) \to \infty$ as $t \to \infty$.

\subsection{Sublinear Regret}
\label{ssec:sublinear-regret}

This subsection is dedicated to the following theorem.

\begin{theorem}[Sublinear Regret]
\label{thm:aoim-implies-sublinear-regret}
If the discount function $\gamma$ satisfies \autoref{ass:gamma},
the environment $\mu \in \M$ satisfies the recoverability assumption, and
$\pi$ is asymptotically optimal in mean, i.e.,
\[
\EE_\mu^\pi \big[ V^*_\mu(\ae_{<t}) - V^\pi_\mu(\ae_{<t}) \big] \to 0
\text{ as $t \to \infty$,}
\]
then $R_m(\pi, \mu) \in o(m)$.
\end{theorem}

If the items in \autoref{ass:gamma} are violated,
\autoref{thm:aoim-implies-sublinear-regret} can fail:
\begin{itemize}
\item If $\gamma_t = 0$ for some time steps $t$,
our policy does not care about those time steps and might take actions
that have large regret.
\item Similarly if $\gamma$ oscillates between high values and very low values:
our policy might take high-regret actions in
time steps with comparatively lower $\gamma$-weight.
\item If the horizon grows linearly,
infinitely often
our policy might spend some constant fraction of the current effective horizon exploring,
which incurs a cost that is a constant fraction of the total regret so far.
\end{itemize}

To prove \autoref{thm:aoim-implies-sublinear-regret},
we apply the following technical lemma.

\begin{lemma}[Value and Regret]
\label{lem:value-and-regret}
Let $\eps > 0$ and
assume the discount function $\gamma$ satisfies \autoref{ass:gamma}.
Let $(d_t)_{t \in \mathbb{N}}$ be a sequence of numbers with
$|d_t| \leq 1$ for all $t$.
If there is a time step $t_0$ with
\begin{equation}\label{eq:ao}
\frac{1}{\Gamma_t} \sum_{k=t}^\infty \gamma_k d_k < \eps
\quad \forall t \geq t_0
\end{equation}
then
\[
     \sum_{t=1}^m d_t
\leq t_0 + \varepsilon(m - t_0 + 1)
     + \frac{1 + \varepsilon}{1 - \varepsilon} H_m(\varepsilon)
\]
\end{lemma}
\begin{proof}
This proof essentially follows
the proof of \cite[Thm.~17]{Hutter:2006discounting};
see \autoref{app:omitted-proofs}.
\end{proof}

\begin{proof}[Proof of \autoref{thm:aoim-implies-sublinear-regret}]
Let $(\pi_m)_{m \in \mathbb{N}}$ denote any sequence of policies,
such as a sequence of policies that
attain the supremum in the definition of regret.
We want to show that
\[
    \EE_\mu^{\pi_m} \left[ \sum_{t=1}^m r_t \right]
    - \EE_\mu^\pi \left[ \sum_{t=1}^m r_t \right]
\in o(m).
\]
For
\begin{equation}\label{eq:def-d_k}
d_k^{(m)} := \EE_\mu^{\pi_m} [ r_k ] - \EE_\mu^\pi [ r_k ]
\end{equation}
we have $-1 \leq d_k^{(m)} \leq 1$
since we assumed rewards to be bounded between $0$ and $1$.
Because the environment $\mu$ satisfies the recoverability assumption
we have
\begin{align*}
&\left|
    \EE^{\pi^*_\mu}_\mu[ V^*_\mu(\ae_{<t}) ] - \EE^\pi_\mu[ V^*_\mu(\ae_{<t}) ]
 \right|
\to 0
\text{ as $t \to \infty$}, \text{ and} \\
&\sup_m \left| \EE^{\pi^*_\mu}_\mu[ V^*_\mu(\ae_{<t}) ]
               - \EE^{\pi_m}_\mu[ V^*_\mu(\ae_{<t}) ]
        \right|
\to 0
\text{ as $t \to \infty$},
\end{align*}
so we conclude that
\[
    \sup_m \left|
      \EE^\pi_\mu[ V^*_\mu(\ae_{<t}) ] - \EE^{\pi_m}_\mu[ V^*_\mu(\ae_{<t}) ]
    \right|
\to 0
\]
by the triangle inequality and thus
\begin{equation}\label{eq:recover}
    \sup_m \EE^{\pi_m}_\mu[ V^*_\mu(\ae_{<t}) ]
    - \EE^\pi_\mu[ V^*_\mu(\ae_{<t}) ]
\to 0
\text{ as $t \to \infty$}.
\end{equation}
By assumption the policy $\pi$ is asymptotically optimal in mean,
so we have
\[
    \EE^\pi_\mu[ V^*_\mu(\ae_{<t}) ] - \EE^\pi_\mu[ V^\pi_\mu(\ae_{<t}) ]
\to 0
\text{ as $t \to \infty$},
\]
and with \eqref{eq:recover} this combines to
\[
    \sup_m \EE^{\pi_m}_\mu[ V^*_\mu(\ae_{<t}) ]
    - \EE^\pi_\mu[ V^\pi_\mu(\ae_{<t}) ]
\to 0 \text{ as $t \to \infty$}.
\]
From $V^*_\mu(\ae_{<t}) \geq V^{\pi_m}_\mu(\ae_{<t})$ we get
\begin{equation}\label{eq:convergence-of-values}
\limsup_{t \to \infty} \left(
    \sup_m \EE^{\pi_m}_\mu[ V^{\pi_m}_\mu(\ae_{<t}) ]
    - \EE^\pi_\mu[ V^\pi_\mu(\ae_{<t}) ]
    \right)
\leq 0.
\end{equation}
For $\pi' \in \{ \pi, \pi_1, \pi_2, \ldots \}$ we have
\begin{align*}
   \EE^{\pi'}_\mu[ V^{\pi'}_\mu(\ae_{<t}) ]
&= \EE^{\pi'}_\mu \left[ \frac{1}{\Gamma_t}
     \EE^{\pi'}_\mu \left[ \sum_{k=t}^\infty \gamma_k r_k \,\middle|\, \ae_{<t} \right]
   \right] \\
&= \EE^{\pi'}_\mu \left[ \frac{1}{\Gamma_t} \sum_{k=t}^\infty \gamma_k r_k \right] \\
&= \frac{1}{\Gamma_t} \sum_{k=t}^\infty \gamma_k \EE^{\pi'}_\mu [ r_k ],
\end{align*}
so from \eqref{eq:def-d_k} and \eqref{eq:convergence-of-values} we get
\[
\limsup_{t \to \infty} \sup_m
\frac{1}{\Gamma_t} \sum_{k=t}^\infty \gamma_k d_k^{(m)} \leq 0.
\]
Let $\eps > 0$ and choose $t_0$ independent of $m$ and large enough such that
$\sup_m \sum_{k=t}^\infty \gamma_k d_k^{(m)} / \Gamma_t < \eps$
for all $t \geq t_0$.
Now we let $m \in \mathbb{N}$ be given and
apply \autoref{lem:value-and-regret} to get
\begin{align*}
      \frac{R_m(\pi, \mu)}{m}
&=    \frac{\sum_{k=1}^m d_k^{(m)}}{m} \\
&\leq \frac{t_0 + \varepsilon(m - t_0 + 1) + \frac{1 + \varepsilon}{1 - \varepsilon} H_m(\varepsilon)}{m}.
\end{align*}
Since $H_t(\eps) \in o(t)$ according to \assref{itm:gamma(c)}
we get $\limsup_{m \to \infty} R_m(\pi, \mu) / m \leq 0$.
\end{proof}

\begin{example}[Converse of \autoref{thm:aoim-implies-sublinear-regret} is False]
\label{ex:sublinear-regret-does-not-imply-ao}
Let $\mu$ be a two-armed Bernoulli bandit with means $0$ and $1$
and suppose we are using geometric discounting
with discount factor $\gamma \in [0, 1)$.
This environment is recoverable.
If our policy $\pi$ pulls the suboptimal arm
exactly on time steps $1, 2, 4, 8, 16, \ldots$,
regret will be logarithmic.
However, on time steps $t = 2^n$ for $n \in \mathbb{N}$
the value difference $V^*_\mu - V^\pi_\mu$
is deterministically at least $1 - \gamma > 0$.
\end{example}

\subsection{Implications}
\label{ssec:regret-implications}

We get the following immediate consequence.

\begin{corollary}[Sublinear Regret for the Optimal Discounted Policy]
\label{cor:discounted-and-undiscounted-regret}
If the discount function $\gamma$ satisfies \autoref{ass:gamma} and
the environment $\mu$ satisfies the recoverability assumption,
then
$
    R_m(\pi^*_\mu, \mu)
\in o(m)
$.
\end{corollary}
\begin{proof}
From \autoref{thm:aoim-implies-sublinear-regret}
since the policy $\pi^*_\mu$ is (trivially) asymptotically optimal in $\mu$.
\end{proof}

If the environment does not satisfy the recoverability assumption,
regret may be linear \emph{even on the optimal policy}:
the optimal policy maximizes discounted rewards and
this short-sightedness might incur a tradeoff
that leads to linear regret later on
if the environment does not allow recovery.

\begin{corollary}[Sublinear Regret for Thompson Sampling]
\label{cor:thompson-sampling-regret}
If the discount function $\gamma$ satisfies \autoref{ass:gamma} and
the environment $\mu \in \M$ satisfies the recoverability assumption,
then $R_m(\pi_T, \mu) \in o(m)$
for the Thompson sampling policy $\pi_T$.
\end{corollary}
\begin{proof}
From \autoref{thm:taxi-ao} and \autoref{thm:aoim-implies-sublinear-regret}.
\end{proof}

\section{Discussion}
\label{sec:discussion}

In this paper we introduced a reinforcement learning policy $\pi_T$ based on Thompson sampling
for general countable environment classes~(\autoref{alg:Thompson-sampling}).
We proved two asymptotic statements about this policy.
\autoref{thm:taxi-ao} states that $\pi_T$ is asymptotically optimal in mean:
the value of $\pi_T$ in the true environment
converges to the optimal value.
\autoref{cor:thompson-sampling-regret} states that
the regret of $\pi_T$ is sublinear:
the difference of the expected average rewards between $\pi_T$
and the best informed policy converges to $0$.
Both statements come without a concrete convergence rate because of the
weak assumptions we made on the environment class.

Asymptotic optimality has to be taken with a grain of salt.
It provides no incentive to the agent to avoid traps in the environment.
Once the agent gets caught in a trap, all actions are equally bad and thus
optimal: asymptotic optimality has been achieved.
Even worse, an asymptotically optimal agent has to explore all the traps
because they might contain hidden treasure.
Overall, there is a dichotomy between the asymptotic nature of
asymptotic optimality and the use of discounting to prioritize the present over the future.
Ideally, we would want to give finite guarantees instead,
but without additional assumptions this is likely impossible in this general setting.
Our regret bound could be a step in the right direction,
even though itself asymptotic in nature.

For Bayesians asymptotic optimality means that
the posterior distribution $w(\,\cdot \mid \ae_{<t})$
concentrates on environments that are indistinguishable from the true environment
(but generally not on the true environment).
This is why Thompson sampling works:
any optimal policy of the environment we draw from the posterior will,
with higher and higher probability,
also be (almost) optimal in the true environment.

If the Bayesian mixture $\xi$ is inside the class $\M$
(as it is the case for the class of
lower semicomputable chronological semimeasures~\cite{Hutter:2005}),
then we can assign $\xi$ a prior probability
that is arbitrarily close to $1$.
Since the posterior of $\xi$ is the same as the prior,
Thompson sampling will act according to the Bayes-optimal policy
most of the time.
This means the Bayes-value of Thompson sampling can be very good;
formally,
$V^*_\xi(\epsilon) - V^{\pi_T}_\xi(\epsilon)$ can be made arbitrarily small,
and thus Thompson sampling can have
near-optimal Legg-Hutter intelligence~\cite{LH:2007int}.

In contrast, the Bayes-value of Thompson sampling can also be very bad:
Suppose you have a class of $(n+1)$-armed bandits indexed $1, \ldots, n$
where bandit $i$ gives reward $1 - \eps$ on arm $1$, reward $1$ on arm $i + 1$,
and reward $0$ on all other arms.
For geometric discounting
and $\eps < (1 - \gamma)/(2 - \gamma)$,
it is Bayes-optimal to pull arm $1$ while Thompson sampling will explore on average
$n/2$ arms until it finds the optimal arm.
The Bayes-value of Thompson sampling is $1/(n-\gamma_{n-1})$
in contract to $(1 - \eps)$ achieved by Bayes.
For a horizon of $n$,
the Bayes-optimal policy suffers a regret of $\eps n$ and
Thompson sampling a regret of $n/2$,
which is much larger for small $\eps$.

The exploration performed by Thompson sampling is qualitatively different from
the exploration by {\BayesExp}~\cite[Ch.~5]{Lattimore:2013}.
{\BayesExp} performs phases of exploration in which
it maximizes the expected information gain.
This explores the environment class completely,
even achieving off-policy prediction~\cite[Thm.~7]{OLH:2013ksa}.
In contrast, Thompson sampling only explores on the optimal policies,
and in some environment classes this will not yield off-policy prediction.
So in this sense the exploration mechanism of Thompson sampling is more reward-oriented
than maximizing information gain.

Possible avenues of future research are
providing concrete convergence rates for specific environment classes
and results for uncountable (parameterized) environment classes.
For the latter, we have to use different analysis techniques because
the true environment $\mu$ is typically assigned a prior probability of $0$
(only a positive density) but
the proofs of \autoref{lem:info-gain} and \autoref{thm:taxi-ao}
rely on dividing by or taking a minimum over prior probabilities.
We also left open
whether Thompson sampling is weakly asymptotically optimal.

\subsubsection*{Acknowledgements}
\autoref{ex:TS-not-sao} was developed jointly with Stuart Armstrong.
We thank Tom Everitt and Djallel Bouneffouf for proofreading.


\bibliographystyle{alpha}
\bibliography{ai}


\cleardoublepage
\appendix

\section{List of Notation}
\label{app:notation}

\begin{xtabular*}{\linewidth}{lp{60mm}}
$:=$
	& defined to be equal \\
$\mathbb{N}$
	& the natural numbers, starting with $0$ \\
$\Delta\mathcal{Y}$
	& the set of all probability distributions on $\mathcal{Y}$ \\
$\X^*$
	& the set of all finite strings over the alphabet $\X$ \\
$\X^\infty$
	& the set of all infinite strings over the alphabet $\X$ \\
$\A$
	& the (finite) set of possible actions \\
$\E$
	& the (finite) set of possible percepts \\
$\alpha, \beta$
	& two different actions, $\alpha, \beta \in \A$ \\
$a_t$
	& the action in time step $t$ \\
$e_t$
	& the percept in time step $t$ \\
$r_t$
	& the reward in time step $t$, bounded between $0$ and $1$ \\
$\ae_{<t}$
	& the history up to time $t-1$, i.e., the first $t - 1$ interactions,
	$a_1 e_1 a_2 e_2 \ldots a_{t-1} e_{t-1}$ \\
$\epsilon$
	& the history of length $0$ \\
$\eps$
	& a small positive real number \\
$\gamma$
	& the discount function $\gamma: \mathbb{N} \to \mathbb{R}_{\geq 0}$ \\
$\Gamma_t$
	& a discount normalization factor,
	$\Gamma_t := \sum_{i=t}^\infty \gamma_i$ \\
$H_t(\eps)$
	& the $\eps$-effective horizon, defined in \eqref{eq:effective-horizon} \\
$\pi$
	& a (stochastic) policy, i.e., a function $\pi: \H \to \Delta\A$ \\
$\pi^*_\nu$
	& an optimal policy for environment $\nu$ \\
$V^\pi_\nu$
	& value of the policy $\pi$ in environment $\nu$ \\
$n, k, i$
	& natural numbers \\
$t$
	& (current) time step \\
$m$
	& time step at the end of an effective horizon \\
$\M$
	& a countable class of environments \\
$\nu, \mu, \rho$
	& environments from $\M$, i.e., functions $\nu: \H \times \A \to \Delta\E$;
	  $\mu$ is the true environment \\
$\xi$
	& Bayesian mixture over all environments in $\M$ \\
\end{xtabular*}

\section{Omitted Proofs}
\label{app:omitted-proofs}

Let $P$ and $Q$ be two probability distributions.
We say $P$ is \emph{absolutely continuous with respect to $Q$} ($P \ll Q$)
iff $Q(E) = 0$ implies $P(E) = 0$ for all measurable sets $E$.
If $P \ll Q$ then there is a function $dP/dQ$ called
\emph{Radon-Nikodym derivative} such that
\[
\int f dP = \int f \frac{dP}{dQ} dQ
\]
for all measurable functions $f$.
This function $dP/dQ$ can be seen as a density function of $P$
with respect to the background measure $Q$.

\begin{proof}[Proof of \autoref{lem:value-bound}]
Let $P$, $R$, and $Q$ be probability measures with $P \ll Q$ and $R \ll Q$
(we can take $Q := P/2 + R/2$),
let $dP/dQ$ and $dR/dQ$ denote their Radon-Nikodym derivative with respect to $Q$,
and let $X$ denote a random variable with values in $[0, 1]$.
Then
\begin{align*}
      \int X dP - \int X dR
&=    \int \left( X \frac{dP}{dQ} - X \frac{dR}{dQ} \right) dQ \\
&\leq \int_A X \left( \frac{dP}{dQ} - \frac{dR}{dQ} \right) dQ \\
\intertext{with $A := \left\{ x \,\middle|\, \frac{dP}{dQ}(x) - \frac{dR}{dQ}(x) \geq 0 \right\}$}
&\leq \int_A \left( \frac{dP}{dQ} - \frac{dR}{dQ} \right) dQ \\
&=    P(A) - R(A) \\
&\leq \sup_A | P(A) - R(A) |
 =    D(P, R)
\end{align*}
From this also follows $\int X dR - \int X dP \leq D(R, P)$,
and since $D$ is symmetric we get
\begin{equation}\label{eq:tvd-bound}
     \left| \int X dP - \int X dR \right|
\leq D(P, R).
\end{equation}

According to \autoref{def:value-function},
the value function is the expectation of the random variable
$\sum_{k=t}^m \gamma_k r_k / \Gamma_t$ that is bounded between $0$ and $1$.
Therefore we can use \eqref{eq:tvd-bound} with
$P := \nu^{\pi_1}(\,\cdot \mid \ae_{<t})$ and
$R := \rho^{\pi_2}(\,\cdot \mid \ae_{<t})$ on the space $(\A \times \E)^m$
of the histories of length $\leq m$
to conclude that
$|V^{\pi_1,m}_\nu(\ae_{<t}) - V^{\pi_2,m}_\rho(\ae_{<t})|$
is bounded by $D_m(\nu^{\pi_1}, \rho^{\pi_2} \mid \ae_{<t})$.
\end{proof}

\begin{proof}[Proof of \autoref{lem:info-gain}]
From Blackwell-Dubins' theorem~\cite{BD:1962}
we get
$D_\infty(\mu^\pi, \xi^\pi \mid \ae_{<t}) \to 0$
$\mu^\pi$-almost surely,
and since $D$ is bounded, this convergence also occurs in mean.
Thus for every environment $\nu \in \M$,
\begin{equation}\label{eq:KL-to-0}
\EE_\nu^\pi \big[ D_\infty(\nu^\pi, \xi^\pi \mid \ae_{<t}) \big] \to 0
\text{ as $t \to \infty$}.
\end{equation}
Now
\begingroup
\allowdisplaybreaks
\begin{align*}
&\phantom{=}~\;
      \EE_\mu^\pi[F^\pi_\infty(\ae_{<t})] \\
&\leq \frac{1}{w(\mu)} \EE_\xi^\pi[F^\pi_\infty(\ae_{<t})] \\
&=    \frac{1}{w(\mu)} \EE_\xi^\pi \left[
        \sum_{\nu \in \M} w(\nu \mid \ae_{<t})
          D_\infty(\nu^\pi, \xi^\pi \mid \ae_{<t})
      \right] \\
&=    \frac{1}{w(\mu)} \EE_\xi^\pi \left[
        \sum_{\nu \in \M} w(\nu) \frac{\nu^\pi(\ae_{<t})}{\xi^\pi(\ae_{<t})}
          D_\infty(\nu^\pi, \xi^\pi \mid \ae_{<t})
      \right] \\
&=    \frac{1}{w(\mu)} \sum_{\nu \in \M} w(\nu) \EE_\nu^\pi \big[
        D_\infty(\nu^\pi, \xi^\pi \mid \ae_{<t})
      \big]
 \to  0
\end{align*}
\endgroup
by \cite[Lem.~5.28ii]{Hutter:2005} since total variation distance is bounded.
\end{proof}

\begin{proof}[Proof of \autoref{lem:value-and-regret}]
By \assref{itm:gamma(a)}
we have $\gamma_t > 0$ for all $t$
and hence $\Gamma_t > 0$ for all $t$.
By \assref{itm:gamma(b)}
have that $\gamma$ is monotone decreasing,
so we get for all $n \in \mathbb{N}$
\[
     \Gamma_t
=    \sum_{k=t}^\infty \gamma_k
\leq \sum_{k=t}^{\mathclap{t+n-1}} \gamma_t
     + \sum_{\mathclap{k=t+n}}^\infty \gamma_k
=    n\gamma_t + \Gamma_{t+n}.
\]
And with $n := H_t(\eps)$ this yields
\begin{equation}\label{eq:t-independent-bound}
     \frac{\gamma_t H_t(\eps)}{\Gamma_t}
\geq 1 - \frac{\Gamma_{t+H_t(\eps)}}{\Gamma_t}
\geq 1 - \eps
>    0.
\end{equation}
In particular, this bound holds for all $t$ and $\eps > 0$.

Next, we define a series of nonnegative weights $(b_t)_{t \geq 1}$ such that
\[
  \sum_{t=t_0}^m d_k
= \sum_{t=t_0}^m \frac{b_t}{\Gamma_t} \sum_{k=t}^m \gamma_k d_k.
\]
This yields the constraints
\[
\sum_{k=t_0}^t \frac{b_k}{\Gamma_k} \gamma_t = 1 \quad \forall t \geq t_0.
\]
The solution to these constraints is
\begin{equation}\label{eq:b_t-solution}
b_{t_0} = \frac{\Gamma_{t_0}}{\gamma_{t_0}},
\text{ and }
b_t = \frac{\Gamma_t}{\gamma_t} - \frac{\Gamma_t}{\gamma_{t-1}}
\text{ for $t > t_0$}.
\end{equation}
Thus we get
\begin{align*}
      \sum_{t=t_0}^m b_t
&=    \frac{\Gamma_{t_0}}{\gamma_{t_0}}
      + \sum_{t=t_0+1}^m \left(
          \frac{\Gamma_t}{\gamma_t} - \frac{\Gamma_t}{\gamma_{t-1}}
        \right) \\
&=    \frac{\Gamma_{m+1}}{\gamma_m}
      + \sum_{t=t_0}^m \left(
          \frac{\Gamma_t}{\gamma_t} - \frac{\Gamma_{t+1}}{\gamma_t}
        \right) \\
&=    \frac{\Gamma_{m+1}}{\gamma_m} + m - t_0 + 1 \\
&\leq \frac{H_m(\varepsilon)}{1 - \varepsilon} + m - t_0 + 1
\end{align*}
for all $\varepsilon > 0$ according to \eqref{eq:t-independent-bound}.

Finally,
\begin{align*}
      \sum_{t=1}^m d_t
&\leq \sum_{t=1}^{t_0} d_t
      + \sum_{t=t_0}^m \frac{b_t}{\Gamma_t} \sum_{k=t}^m \gamma_k d_k \\
&\leq t_0
      + \sum_{t=t_0}^m \frac{b_t}{\Gamma_t} \sum_{k=t}^\infty \gamma_k d_k
      - \sum_{t=t_0}^m \frac{b_t}{\Gamma_t} \sum_{k=m+1}^\infty \gamma_k d_k \\
\intertext{
and using the assumption \eqref{eq:ao} and $d_t \geq -1$,
}
&<    t_0
      + \sum_{t=t_0}^m b_t \varepsilon
      + \sum_{t=t_0}^m \frac{b_t \Gamma_{m+1}}{\Gamma_t} \\
&\leq t_0 + \frac{\varepsilon H_m(\varepsilon)}{1 - \varepsilon}
      + \varepsilon (m - t_0 + 1)
      + \sum_{t=t_0}^m \frac{b_t \Gamma_{m+1}}{\Gamma_t}
\end{align*}
For the latter term we substitute \eqref{eq:b_t-solution} to get
\begin{align*}
     \sum_{t=t_0}^m \frac{b_t \Gamma_{m+1}}{\Gamma_t}
&=   \frac{\Gamma_{m+1}}{\gamma_{t_0}} + \sum_{t=t_0 + 1}^m
     \left( \frac{\Gamma_{m+1}}{\gamma_t} - \frac{\Gamma_{m+1}}{\gamma_{t-1}} \right) \\
&=   \frac{\Gamma_{m+1}}{\gamma_m}
\leq \frac{H_m(\eps)}{1 - \eps}
\end{align*}
with \eqref{eq:t-independent-bound}.
\end{proof}

\end{document}